\documentclass[sigconf]{acmart}
\usepackage{subcaption}
\usepackage{xspace}
\usepackage[utf8]{inputenc} % allow utf-8 input
\usepackage[T1]{fontenc} % use 8-bit T1 fonts
\usepackage{hyperref}  % hyperlinks
\usepackage{url}  % simple URL typesetting
\usepackage{booktabs}  % professional-quality tables
\usepackage{amsfonts}  % blackboard math symbols
\usepackage{nicefrac}  % compact symbols for 1/2, etc.
\usepackage{bm}
\usepackage{amsmath}
\usepackage{amsthm}
\usepackage{microtype}  % microtypography
\usepackage{environ}
\usepackage{multicol}
\usepackage{multirow}
\usepackage{graphicx} 
\usepackage{tikz}
\usepackage{wrapfig}
\usepackage{makecell}
\usepackage{diagbox}
\usepackage{color}
\usepackage{threeparttable}
\usepackage{pifont}
\usepackage{enumerate}
\usepackage[utf8]{inputenc}
\usepackage{mathtools}
\usepackage{algorithm2e}
\usepackage{enumitem}
\RestyleAlgo{ruled}
\usepackage{cleveref}

\crefname{section}{§}{§§}
\Crefname{section}{§}{§§}

\newtheorem{remark}{Remark}
\newtheorem{definition}{Definition}

\newtheorem{corollary}{Corollary}

\copyrightyear{2024}
\acmYear{2024}
\setcopyright{rightsretained}
\acmConference[KDD '24]{Proceedings of the 30th ACM SIGKDD Conference on
Knowledge Discovery and Data Mining}{August 25--29, 2024}{Barcelona, Spain}
\acmBooktitle{Proceedings of the 30th ACM SIGKDD Conference on Knowledge Discovery and Data Mining (KDD '24), August 25--29, 2024, Barcelona, Spain}\acmDOI{10.1145/3637528.3671911}
\acmISBN{979-8-4007-0490-1/24/08}

\makeatletter
\gdef\@copyrightpermission{
 \begin{minipage}{0.3\columnwidth}
 \href{https://creativecommons.org/licenses/by/4.0/}{\includegraphics[width=0.90\textwidth]{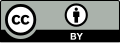}}
 \end{minipage}\hfill
 \begin{minipage}{0.7\columnwidth}
 \href{https://creativecommons.org/licenses/by/4.0/}{This work is licensed under a Creative Commons Attribution International 4.0 License.}
 \end{minipage}
 \vspace{5pt}
}
\makeatother

\begin{document}
% \fancyhead{}

\title{Logical Reasoning with Relation Network for Inductive Knowledge Graph Completion}

\author{Qinggang Zhang}
\affiliation{
  \institution{The Hong Kong Polytechnic University}
  \city{Hung Hom}
  \country{Hong Kong SAR}
  }
\email{qinggangg.zhang@connect.polyu.hk}

\author{Keyu Duan}
\affiliation{%
  \institution{National University of Singapore}
  \city{21 Lower Kent Ridge Road}
  \country{Singapore}
  }
% \email{k.duan@u.nus.edu}
\email{k.duan@u.nus.edu}

\author{Junnan Dong}
\affiliation{%
  \institution{The Hong Kong Polytechnic University}
  \city{Hung Hom}
  \country{Hong Kong SAR}
  }
\email{hanson.dong@connect.polyu.hk}

\author{Pai Zheng}
\affiliation{%
 \institution{The Hong Kong Polytechnic University}
 \city{Hung Hom}
  \country{Hong Kong SAR}
  }
\email{pai.zheng@polyu.edu.hk}

\author{Xiao Huang}
% \authornote{Corresponding author}
\affiliation{%
 \institution{The Hong Kong Polytechnic University}
\city{Hung Hom}
  \country{ Hong Kong SAR}
  }
\email{xiaohuang@comp.polyu.edu.hk}

\renewcommand{\shortauthors}{Qinggang Zhang, Keyu Duan, Junnan Dong, Pai Zheng, \& Xiao Huang}

% \renewcommand{\shortauthors}{Qinggang Zhang et al.}
% \shortauthors{<short author list for running head>}
\begin{abstract}

Inductive knowledge graph completion (KGC) aims to infer the missing relation for a set of newly-coming entities that never appeared in the training set. Such a setting is more in line with reality, as real-world KGs are constantly evolving and introducing new knowledge. Recent studies have shown promising results using message passing over subgraphs to embed newly-coming entities for inductive KGC. However, the inductive capability of these methods is usually limited by two key issues. $(i)$ KGC always suffers from data sparsity, and the situation is even exacerbated in inductive KGC where new entities often have few or no connections to the original KG. $(ii)$ Cold-start problem. It is over coarse-grained for accurate KG reasoning to generate representations for new entities by gathering the local information from few neighbors. To this end, we propose a novel i\underline{N}f\underline{O}max \underline{R}el\underline{A}tion \underline{N}etwork, namely NORAN, for inductive KG completion. It aims to mine latent relation patterns for inductive KG completion. Specifically, by centering on relations, NORAN provides a hyper view towards KG modeling, where the correlations between relations can be naturally captured as entity-independent logical evidence to conduct inductive KGC. Extensive experiment results on five benchmarks show that our framework substantially outperforms the state-of-the-art KGC methods.

\end{abstract}


\maketitle

\section{Introduction}
Knowledge  graphs (KGs) are a collection of factual information to describe a certain field of the real world~\cite{Bollacker-etal07Freebase,Lehmann-etal15DBpedia,Mahdisoltani-etal15YAGO,li2022kcube,li2022constructing,dong2023hierarchy}, which are often formulated as $(h, r, t)$ indicating there is a relation $r$ between the head entity $h$ and tail entity $t$.
Since KGs are usually incomplete~\cite{zhang2022contrastive,dong2023active}, as well as the expensive costs to collect facts manually, how to automatically perform knowledge graph completion (KGC) has been a widely-studied topic in the KG community~\cite{chen2020review,zhang2023integrating,ji2021survey}.
Currently, embedding-based methods~\cite{bordes2013translating,yang2014embedding,lin2015learning,trouillon2016complex,kazemi2018simple} play a dominant role in KGC. The typical methods embed the entities and relations into low-dimensional spaces with various objectives and then predict the relation between any two entities. However, most embedding-based methods are merely applicable under transductive settings, where KGC models predict the missing relations among entities that are already present in the training set. When some unseen entities are newly added during testing, they have to retrain the whole KG embeddings for accurate inference, which is not feasible in practice due to the high update frequency and large sizes of real-world KGs.

Recently, there has been increasing interest in KGC under inductive setting~\cite{wang2019logic,galkin2021nodepiece,wang2021relational}, where KGC models are used to infer the relations for newly-coming entities that have never appeared in the training set. The inductive KGC is more reflective of real-world scenarios, as KGs always keep introducing new information, such as new users and products in e-commerce KGs and new molecules in biomedical KGs. 
Some efforts have been made to obtain the inductive embeddings for new entities using external resources, such as entity attributes~\cite{hao2020inductive,qin2020generative}, textual descriptions~\cite{shi2018open,tang2019knowledge} and ontological schema~\cite{geng2021ontozsl}. However, these supplemental sources are often prohibitive to acquire, impeding their success in practice. An alternative solution is to induce entity-independent rules from the KG in either statistical~\cite{galarraga2013amie,meilicke2018fine} or  differentiable manners~\cite{yang2014embedding,yang2017differentiable,sadeghian2019drum}.  It regards inductive KG completion as a rule mining problem, i.e. predicting the relationship $r$ for any entity pair $(h, t)$ based on predefined or adaptively learned rules. However, these methods often suffer from scalability and generalizability issues since different KGs have distinct rules with domain knowledge involved. 
% and lack the capability of learning expressive representations for real-world datasets since different KGs have distinct rules with domain knowledge involved. 

As the prosperity of graph neural networks (GNNs) and their strong inductive bias~\cite{zhou2023interest,chen2022neighbor,chen2020label,zhou2021temporal,chen2020graph}, many works extend message-passing (MP) strategy used in GNNs to multi-relational graphs for inductive KGC~\cite{schlichtkrull2018modeling,teru2020inductive,wang2021relational,wang2022exploring}. GraIL~\cite{teru2020inductive} is one of the pioneer work, which implicitly learns logical rules by reasoning over enclosing subgraphs. Specifically, GraIL first extracts an enclosing subgraph surrounding the target triple, and then generate the representations of entities by aggregating the structural information of its neighbors in the extracted subgraph. GraIL and its follow-up methods~\cite{chen2021topology,liu2021indigo,ontologyschema,lee2023ingram,jininductive,cui2022inductive,zhang2018link,zhang2020revisiting} have shown promising results using message passing over subgraphs for inductive KGC. However, the inductive capability of these methods is usually limited by two key issues. $(i)$ KGC always suffers from data sparsity, and the situation is even more difficult in inductive KGC, where the new emerging entities usually have few or even no connections to the original KG. $(ii)$  Existing message passing methods assume that messages are associated with entities and are passed from nodes to nodes iteratively~\cite{xu2018powerful}. Such entity-oriented message-passing strategy weakens the role of relation semantics~\cite{wang2021relational}, which violates the nature of inductive KGC since the learned logical evidence should be entity-independent.

In real-world scenarios, the enclosing subgraph of a particular triple in the KG contains the logical evidence needed to deduce the relation between the target nodes~\cite{Teru2020InductiveRP}.  
Taking Figure~\ref{fig:illustration} $(i)$ as an example, learning the relation pattern among \textit{Elon Musk}, \textit{TESLA}, and \textit{California State} can help the KGC model make a confident prediction of \textit{(Martin Eberhard, :Livein, California State)} 
 since ``\textit{:LiveIn} $\simeq$ \textit{:WorkAt} $\land$ \textit{:LocatedIn}''. To this end, we propose a novel message passing neural framework, i.e., NORAN, which aims to mine relation semantics for inductive KG completion. 
 Specifically, we first provide a deep insight of formulating the problem to bridge the gap between embedding-based paradigm and inductive setting.
 Inspired by the findings, we propose \textit{relation network}, a novel relation-induced graph derived from the original KG. 
The reconstructed relation network provides a hyper view towards KG modeling by centering relations. Its topological structure clearly reflects the distribution of relations, i.e., each node in the relation network represents a relational instance and the links among the nodes reflect the relation correlation. Thus, modeling the relation network via message-passing networks can naturally capture relation patterns as entity-independent context information to conduct inductive KGC. In this paper, we formally define such context information as logic evidence and reveal its correlations with our relation network. 
Our main contribution is summarized as follows:
% \vspace{-3mm}
\begin{itemize}
   \item We propose a novel  framework, i.e., NORAN, which aims to mine latent relation semantics for inductive KG completion.
    \item 
    % A novel hyper view is provided towards KG modeling by centering on relations, where correlational patterns can be naturally captured as entity-independent contextual semantics for inductive KGC.
    Centering on relations, we propose a hyper view towards KG modeling, i.e., relation network, and formally define the inductive KGC as \textit{k-hop logic reasoning} over the hypergraph.
     \item 
     We propose a informax training objective to implicitly capture logic evidence from relation network for effective KGC.
    \item We conduct a detailed theoretical analysis to explore which kind of message-passing strategy is more effective and provide guidelines for model selection to enable inductive reasoning over the reconstructed relation network.

    \item 
    Extensive experiments on five real-world KG benchmarks demonstrate the superiority of our proposed NORAN.
\end{itemize}
% }

\section{Preliminaries} \label{sec:related_work}
\subsection{Transductive KG Completion}
Given any pair of entities $(h, t)$, traditional KG completion model aims to predict the probability of relation $r$ linking $h$ and $t$, i.e. $p(r|h,t)$. The relation type with the highest probability will be the output after ranking. Some related work also equally formalizes the problem as a relation classification problem, which predicts the relation type with ${Softmax}_r(p(r|h,t))$.

Currently, the embedding-based paradigm dominates the transductive KGC problem. TransE~\cite{bordes2013translating}, 
DistMult~\cite{yang2014embedding}, TransR~\cite{lin2015learning}, ComplEx~\cite{trouillon2016complex}, and SimplE~\cite{kazemi2018simple} are the representative embedding-based methods, which embed entity and relations into RHS space with carefully designed modeling strategy. Embedding-based methods encompass two branches: $(i)$ translating-based methods, which interpret relations as translation operations on the low-dimensional embeddings of entities; $(ii)$ semantic matching methods, which compute the similarity between relations and entities with their semantic information instantiated by their low-dimensional representations. Despite its capability to learn expressive representations, the embedding-based paradigm is not applicable under inductive settings since they have to retrain the whole KG embeddings when some unseen entities are newly added during testing.

\subsection{Inductive KG Completion}
Different from transductive KGC, inductive KGC tries to predict the missing relations for a set of new entities that are not present in the training set. Given a knowledge graph $\mathcal{G} = \{\mathcal{E}, \mathcal{R}\} $  and a set of incoming entities $\mathcal{E}_{u} = \{ e \ |\ e \notin \mathcal{E}  \}$, our goal is to predict the missing relations among $\mathcal{E}$ and $\mathcal{E}_u$. The prediction task contains two scenarios: $(i)$ predict relations between $e_1 \in \mathcal{E}$ and $e_2 \notin \mathcal{E}$ with $P(r|e1,e2)$, namely \textit{inductive relation prediction}; $(ii)$ predict relations within unseen entities ($e \notin \mathcal{E}$), namely \textit{fully-inductive relation prediction}. Both two tasks emphasize on learning entity-independent logic evidence and typical representations of relations. In this work, we focus on the former setting, \textit{inductive relation prediction}, since the latter is within the scope of \textit{transfer learning} that leverages the information from one KG to the others.

\begin{figure*}[htp]
	\centering
	\includegraphics[width=0.94\textwidth]{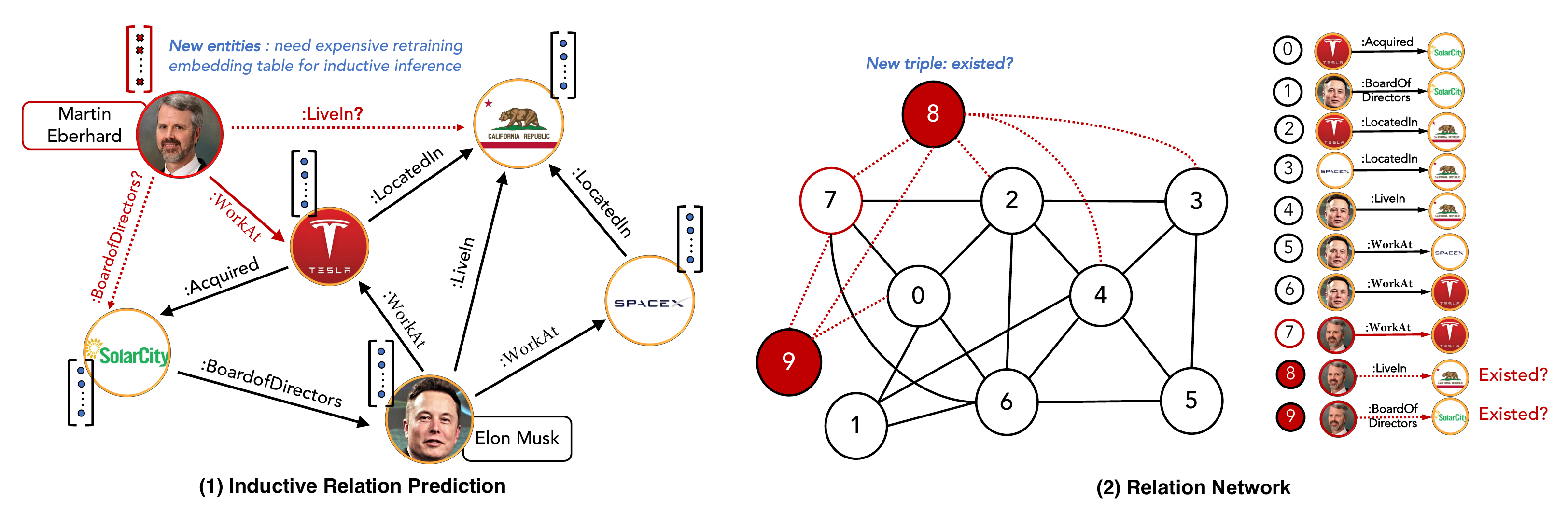}
	\vspace{-5mm}
	\caption{\textit{(i)} A toy example of inductive knowledge graph completion, i.e. predicting the relationship (``?'') for unseen entities, e.g. ``Martin Eberhard'', provided with a few links, e.g. (Martin, :WorkAt, TESLA); \textit{(ii)} Illustration of the corresponding \textit{relation network} for knowledge graph, which regards each triple as a relational node and thus could aggregate context information for inductive inference without expensive retraining look-up embedding tables as embedding-based paradigm.}\label{fig:illustration}
	\vspace{-5mm}
\end{figure*}

% {\qing
Under inductive settings, most of the previous effective works can be summarized as rule-induction based paradigm~\cite{yang2014embedding,yang2017differentiable,sadeghian2019drum,shengyuan2024differentiable}, which follows explicit or implicit logic evidence for knowledge graph reasoning. 
It is worth noting that this paradigm is compatible with embedding-based strategy, and most rule-induction-based methods are on the branch of the aforementioned \textit{semantic matching methods}. More recently, with the prosperity of GNNs, many works~\cite{wang2021relational,teru2020inductive,schlichtkrull2018modeling,yan2022cycle} extend the MP strategy to knowledge graphs. R-GCN~\cite{schlichtkrull2018modeling} firstly proposed a graph convolution operation for relational data. GraIL~\cite{teru2020inductive} and its follow-up methods~\cite{cui2022inductive,zhang2018link,ontologyschema,lee2023ingram,jininductive,zhang2020revisiting,geng2023relational,lin2022incorporating} applied the linking prediction techniques, i.e. subgraph extraction and labeling trick, to tackle the inductive KGC problem. 
% PathCon~\cite{wang2021relational} proposed a node-centered relational message passing strategy and combined it and meta-path-based reasoning together for KGC. 
MP-based methods benefit from their inductive bias but still require dedicated efforts to design the relational message-passing strategy with various options of combination and aggregation functions. By contrast, provided with extensive existing SOTA GNN architectures, our proposed framework is model-agnostic to take full advantage of them. Furthermore, the significance of relation semantics has been concurrently acknowledged in several recent studies~\cite{ontologyschema,lee2023ingram,jininductive}.
Notably, the relation network in our research diverges substantially from the relation graph used in these work. They introduced relation schema to quantify the interconnectivity among various types of relations for inductive KGC with new relations, while our proposed relation network aims to capture the contextual semantics between every pair of relational triples for KGC with previously unseen entities.

\begin{table}[!t]
	\caption{Unified notations for crucial terms.}
	% \vspace{-3mm}
	\label{tab:notations}
	\centering
	\begin{tabular}{lr}
		\toprule
		Notation                    & Description                                   \\
		\midrule
		$\mathcal{G}$             & a knowledge graph                             \\
		$(h,r,t)$                 & (head entity, relation, tail entity)                   \\
		$\widetilde{\mathcal{G}}$ & a relation network                            \\
		$\Gamma(\cdot)$           & the combination function in relation network \\
		% $\Lambda(\mathcal{G})$  & the $k$-hop logic evidence of $\mathcal{G}$      \\
		$\Psi^k$, \(\Omega^k\)    & $k$-layer GNN models                          \\
		\bottomrule
	\end{tabular}
	\vspace{-5mm}
\end{table}

\section{Problem Formulation}\label{sec:problem_formulation}
We unify the main notations used in this work and provide a brief view in Table~\ref{tab:notations}. 
We provide an insight into the gap between embedding-based paradigms and \textit{inductive} settings. Given a KG $\mathcal{G}=\{(h, r, t) \ |\ h,t\in \mathcal{E}, r\in \mathcal{R} \}$, the embedding-based paradigms formalize their model as $p(r|h,t)$, where predict the relationship $r$ with the prior knowledge of entities $h$ and $t$. However, under \textit{inductive} setting, the new entities trigger non-trivial perturbation to the distribution of the prior knowledge graph, thus biasing the model $p(r|h,t)$ from the prior ground truth $p(h,t)$. To this end, we propose a new formulation for the general KGC problem:
\begin{equation} \small
	p(h,r,t) = p(r|h,t)\cdot p(h,t),
\end{equation}
where we predict triple existence, i.e. $p(h,r,t)$, rather than the conditional probability $p(r|h,t)$. Such a formulation has several profits. Firstly, $p(h,r,t)$ is capable of capturing the evolving distribution of $p(h,t)$ under \textit{inductive} setting since no prior distribution is required. Meanwhile, for \textit{transductive} setting, where $p(h,t)$ is static, $p(h,r,t)$ can be collapsed to the previous modeling $p(r|h,t)$. From another perspective, $p(h,r,t)$ encompasses two crucial terms, relation type $p(r|h,t)$ and relation existence $p(h,t)$. Previous modeling $p(r|h,t)$ predicts the relation type $r$ between an entity pair $(h,t)$, which is identical to a classification problem. However, an ideal KGC model should not only predict the relation type but also induce the incomplete entity pairs $(h,t)$, while the latter one is unfortunately ignored in previous modeling. To this end, we propose \textit{relation network}, a novel view towards knowledge graph modeling. It augments the target KG into hyper-views by regarding each relational triple as nodes. Following this, the link prediction task, i.e., $p(r|h,t)$,  can be naturally transformed into predicting the existence of the corresponding node in the reconstructed relation network, i.e., $p(h,r,t)$.

\section{Methodology}\label{sec:methodology}
% {\qing 
In this section, we propose NORAN to mine latent pattern among relational instances for inductive KG completion. As shown in Figure~\ref{fig:illustration}, our proposed NORAN consists of three key components: $(i)$ Relation network construction. It provides a new modeling of KGs by centering on relations. $(ii)$ Relational message passing, a relational message passing framework is defined to implicitly capture entity-independent logical evidence from the reconstructed relation network. $(iii)$ Training scheme and KG inference. Furthermore, we also provide a detailed theoretical analysis to explore which kind of MP strategy is more effective and provide guidelines for model selection to enable inductive reasoning over the reconstructed relation network.

\subsection{Relation Network Construction}\label{sec:relation_network}
Traditional KG representation learning methods usually model the KG as a heterogeneous graph by centering on the entities and regarding relations as semantic edges. However, it is hard to capture complex semantics from KG relations and further learn expressive representations for inductive KGC. To fill the gap between traditional KG representation learning paradigm and inductive KGC, in this work, we value the correlations between relations, accordingly provide a novel hyper view towards KG modeling with relations, and formally define the inductive KG completion as  \textit{k-hop logic reasoning} over an entity-independent network.

\begin{table*}[!ht]
	\caption{The general description framework for MP layers}
	\vspace{-5mm}
	\centering
	\resizebox{0.90\textwidth}{!}{
		\begin{tabular}{cclcc}
			\toprule
			MP layer                                & Motivation          & Convolutional Matrix $\mathcal{C}$                                                & Feature Transformation $f$                       & Category of $\mathcal{C}$ \\
			\midrule
			GCN~\cite{kipf2016semi}                & Spatial Convolution & $\mathcal{C} = \widetilde{\bm{D}}^{1/2}{(\bm{A}+\bm{I})}\widetilde{\bm{D}}^{1/2}$ & $f = \bm{W} \in \mathbb{R}^{d_k \times d_{k+1}}$ & Fixed                     \\
			GraphSAGE~\cite{hamilton2017inductive} & Inductive Learning  & $\mathcal{C} = \widetilde{\bm{D}}^{-1}(\bm{A}+\bm{I})$                            & $f = \bm{W} \in \mathbb{R}^{d_k \times d_{k+1}}$ & Fixed                     \\
			GIN~\cite{xu2018powerful}              & WL-Test             & $\mathcal{C} = \bm{A}+\bm{I}$                                                     & $f$ is a two-layer MLP                           & Fixed                     \\
			SGC~\cite{wu2019simplifying}           & Scalability         & $\mathcal{C} = \widetilde{\bm{D}}^{1/2}{(\bm{A}+\bm{I})}\widetilde{\bm{D}}^{1/2}$ & $f = lambda\ x: x$                                     & Fixed                     \\
			\midrule
			GAT~\cite{velivckovic2017graph}        & Self Attention      & $\begin{cases} \mathcal{C} = (\bm{A}+\bm{I})\cdot \mathcal{T} \\ \mathcal{T}_{(i,j)} = \frac{exp([\Theta \bm{x}_i || \Theta \bm{x}_j]\cdot \bm{a})}{\sum_{k \in \mathcal{N}(i) \cup i} exp([\Theta \bm{x}_i || \Theta \bm{x}_j]\cdot \bm{a})}  \end{cases}$                                                      & $f = \Theta \in \mathbb{R}^{d_k \times d_{k+1}}$ & Learnable                 \\
			\bottomrule
			\multicolumn{5}{l}{\small  $^a$ $\widetilde{\bm{D}}$ is the diagonal matrix of $(\bm{A}+\bm{I})$}
		\end{tabular}}
		\vspace{-5mm}
	\label{tab:framework_for_mp_1}
\end{table*}
\subsubsection{Relation network}
To better capture the latent logic evidence for effective reasoning, we first propose a novel view towards knowledge graph modeling that augments the target KG into hyper-views by compressing each triple into a relational node. Concretely, we perform a relation-induced construction process to build the \emph{relation network}. In essence, we follow one strategical criterion: regard each triple as relational nodes and link them in the relational graph if they share the same entity in the original KG. 

\begin{definition}\label{def:relational_graph}
	\textbf{Relation network.} Given a knowledge graph $\mathcal{G}=\{ (h,r,t) | h,t \in \mathcal{E}, r \in \mathcal{R}\}$, the corresponding relation network is the triple-level graph $\mathcal{G}_r = (\widetilde{\mathcal{V}}, \widetilde{\mathcal{A}}, \widetilde{\mathcal{X}_{r}})$, where $\widetilde{\mathcal{V}}$ and $\widetilde{\mathcal{A}}$ are the set of relational nodes and adjacency metric respectively. $\widetilde{\mathcal{X}_{r}} = \{ \Gamma(h,r,t) \ |\ (h,r,t) \in \mathcal{G} \}$ represents the feature matrix of $\widetilde{\mathcal{V}}$, where $\Gamma(\cdot)$ is the concatenation function that transforms each relational triple into a node, i.e., $v = \Gamma(h,r,t)$. $\widetilde{\mathcal{A}}(v, u|v,u \in \widetilde{\mathcal{V}}) = 1$, if $v$ and $u$ share the same entity in the original KG.
 % %\vspace{-2mm}
\end{definition}

According to this criterion, there are two latent construction particulars corresponding to two construction issues: $(i)$ the semantics of various `\textit{entity sharing}' patterns (Remark~\ref{remark: linking patterns} in Section~\ref{sec:ablation_relation_network}) and $(ii)$ the semantics of the `\textit{linking direction}' (Remark~\ref{remark: linking directions} in Section~\ref{sec:ablation_relation_network}). We conduct a detailed ablation study in Section~\ref{sec:ablation_relation_network} to explore different ways to construct the relation network.

\subsubsection{Interpretation with logic reasoning}~\label{sec:def_logicrule}
The topological graph structure of the relation network clearly reflects the distribution of relations and the link between different nodes reflects the relation correlations. Thus, modeling the relation network via a k-layer message-passing network can naturally capture relation patterns as entity-independent context information to conduct inductive KGC. In this paper, we formally define such context information as the \textit{k-hop logic evidence} over the relation network.

\begin{definition}\label{def:logic_rule}
	\textbf{K-hop logic evidence.} Let $\widetilde{\mathcal{G}} = (\widetilde{\mathbf{V}}, \widetilde{\mathcal{A}})$ be the relation network of a knowledge graph $\mathcal{G}=(\mathcal{E}, \mathcal{R})$. 	Given any center node $v_i \in \widetilde{\mathbf{V}}$, the k-hop ego graph $\widetilde{\mathcal{G}_i}=(\widetilde{\mathbf{V}}_i, \widetilde{\mathbf{A}}_i)$ centering at $v_i$ contains the relational contextual information for logical reasoning. In this paper, we model such contextual information as $\Lambda^{k}(v_i)$ via a k-layer message-passing network $\Omega$. Thus, the overall logic evidence of knowledge graph $\mathcal{G}$ can be modeled by traversing all k-hop ego graphs, i.e., $\Lambda(\mathcal{G}) = \Omega_{GNN}(\Lambda^{k}(v)\ |\ v \in \widetilde{\mathcal{G}})$. 

\end{definition}

Taking a logic evidence ``\textit{:Livein = :WorkAt $\land$ :Locatedin}'' in Fig.~\ref{fig:illustration}~$(i)$ as an instance, the relation pattern among \textit{Elon Musk}, \textit{TESLA}, and \textit{California State} is a typical instantiation of the logic evidence and the triple \textit{(Martin Eberhard, :LiveIn, California State)} can be predicted with relative certainty. Mapping to relation network of Fig.~\ref{fig:illustration}~$(ii)$, according to Definition~\ref{def:logic_rule}, such logic evidences are transformed into $k$-hop ego graphs centering on target nodes, which describe the relational context information for reasoning.

Generally, there are two crucial benefits of constructing relation network: $(i)$ Relation network provides a novel view towards KG modeling, from which entity-independent logic evidence can be naturally modeled via $k$-hop ego graphs that sampled from the relation network to conduct inductive KGC, without fine-grained embeddings for unseen entities.
$(ii)$ Compared with original KG, the relation network shows a more dense structure and is compatible with any GNN models. 

\subsection{Relational Message Passing}\label{sec:logic_rule_infomax}

To capture entity-independent contextual information as logical evidence, we propose a novel relational message-passing framework to conduct inductive KGC.

\subsubsection{Feature initialization}\label{sec:instantiations_of_gamma} 
Since every instance in the relation network is transformed from a corresponding triple $(h, r, t)$ in the original KG, we first randomly initialize the embedding of entities and relations in the original KG, and then adopt a local information modeling layer, i.e, a set of Bi-LSTM units, to learn the local relational structure within each triple:
\begin{equation} \small\small\label{equ:lstm}
	\bm{x}_i = \Gamma(h,r,t) = concat(\Phi(\bm{e}_h, \bm{e}_r, \bm{e}_t)),
\end{equation}
where $\Phi(\cdot)$ is a Bi-LSTM unit, and $\bm{e}_h, \bm{e}_r, \bm{e}_t$ are the initial embeddings of $h$, $r$, and $t$. The output triple embedding $\bm{x}_i$ could well capture the relational structure within the input triple. Thus, we used it as the initial node embedding in the relation network. Note that entity embeddings are \textit{fixed} during the training and inference.

\subsubsection{Message passing}
Message-passing layer is formulated as:
\begin{equation} \small
	\bm{X}^{(k+1)} = \sigma \big(\mathcal{C}^{(k)}\bm{X}^{(k)} \circ f^{(k)}\big),
\end{equation}
where $\bm{X}^{(k)}$ is the relational embedding at the k-th layer; $\mathcal{C}^{(k)}$ is the convolutional matrix for $k$-th layer; $f^{(k)} : \mathbb{R}^{d_k} \rightarrow \mathbb{R}^{d_{k+1}}$ is the linear transformation matrix; and $\sigma$ is the activation function.  
The reformulation of MP layers can be found on Table~\ref{tab:framework_for_mp_1}. Message passing over the relational instances enables the natural capture of relation patterns as entity-independent context information, which is crucial for conducting inductive KGC.

In this paper, we train two message-passing GNNs, i.e., $\Omega$ and $\Psi$: $\Omega$ is applied on the k-hop ego graph to extract logical evidence, while $\Psi$ is applied on \textit{relation network} to learn relational embedding for inductive inference.

\subsubsection{Mutual information maximization}
Reasoning with logical evidence is crucial for inductive KGC, which imitates the inference process of human beings. However, as aforementioned, logic evidence is notoriously sophisticated for representation learning and is usually utilized to regularize reasoning in previous practice. In light of this, we propose \textit{Logic Evidence Information Maximization} (LEIM). At a high level, it is a training objective that tries to maximally preserve mutual information between logic evidence learned from k-hop ego graphs and relational semantics of center node.
% In this part, we proposed a new training objective, i.e., \textit{Logic Rule Information Maximization} (LEIM), that tries to maximally preserve the mutual information between logic evidence and semantics of each relational instance learned from the relation network.
\begin{equation} \small\small\label{equ:LEIM}
	\mathcal{L}_{\omega, \psi} = -\mathcal{I}_{\omega}(\Lambda(\mathcal{G}), \Psi^k(\widetilde{\mathcal{G}})),
\end{equation}
where $\Lambda(\mathcal{G}) = \Omega_{GNN}(\Lambda^{k}(v)\ |\ v \in \widetilde{\mathcal{G}})$ is the learned logic evidence,
 $\mathcal{I}(\cdot, \cdot)$ is the mutual information; $\Psi^k(\cdot)$ is a $k$-layer GNN; and $\omega, \psi$ are the trainable parameters for $\mathcal{I}$ and $\Psi$, respectively. The loss is minimized to find the optimal parameters $\hat{\psi}$ that maximally preserve the mutual information between logic evidence $\Lambda(\mathcal{G})$ and the relational embedding of target node $v$ learned with $\Psi^k(\widetilde{\mathcal{G}})$.
In practice, we adopt a Jensen-Shannon MI estimator (JSD)~\cite{hjelm2018learning,nowozin2016f,zhu2020transfer}:
\begin{multline}\label{equ:jsd_mi}
	\mathcal{I}^{JSD}_{\omega, \psi}(\Lambda, \Psi) = \mathbb{E}_\mathbb{P}\Big(-sp(-T_{\omega, \psi}(\Lambda^k(v), \Psi^k(v)))\Big) \\ - \mathbb{E}_{\mathbb{P}\times\mathbb{P}'}\Big(-sp(T_{\omega, \psi}(\Lambda^k(v'), \Psi^k(v)))\Big),
\end{multline}
where $v$ is the target node sampled from the relation network according to distribution $\mathbb{P}$, $v'$ is sampled from negative distribution $\mathbb{P}'=\mathbb{P}$, and $sp$ denotes softplus activation function. To encourage positive pairs, e.g. $(\Lambda(v), \Psi(v))$, and distinguishes negative pairs, e.g. $(\Lambda(v'), \Psi(v))$, where $v \neq v'$, we introduce $T_{\omega, \psi}(\Lambda, \Psi)$ to predict the correlation between $\Lambda$ and $\Psi$ inspired by GNN pretraining~\cite{zhu2020transfer,hu2020gpt}, formulated as: 
\begin{equation} \small
	T_{\omega, \psi}(\Lambda^k(v), \bm{x}_v) = \sum_{(p,q) \in \Lambda^k(v)}\log (f_\omega([\bm{h}_p || \bm{h}_q || \bm{x}_v])),
\end{equation}
where $\bm{h}_p$ and $\bm{h}_q$ are the output features of the source node $p$ and the target node $q$ generated from the k-hop logical ego graph $\Lambda^k(\mathcal{G})$, $\bm{x}_v = \Psi^k(v)$ is the relatioinal embedding of node $v$. 
% Note that the nodes in $\Lambda^k(\mathcal{G})$ corresponds to the relations in the original KG $\mathcal{G}$.
$f:\mathbb{R}^{ |\bm{h}| + |\bm{h}| + |\bm{x}| } \to \mathbb{R}^{(0,1)}$ is a fully-connected discriminator to distinguish the positive pairs and negative pairs.

\begin{algorithm}[!t]
	\caption{Logical Reasoning with Relation Network}\label{alg:LEIM}
	\SetKwComment{Comment}{// }{}
	\SetKwInput{KwInit}{Init}
	\newcommand{\algrule}[1][.2pt]{\par\vskip.5\baselineskip\hrule height #1\par\vskip.5\baselineskip}
	\SetKw{Inference}{Inductive inference}
	\Comment{All notations are consistent with Table~\ref{tab:notations}}
	\KwIn{a Knowledge Graph $\mathcal{G}$}
	\KwOut{a combination function $\Gamma$ and a GNN model $\Psi$}
	\KwInit{create the relation network $\widetilde{\mathcal{G}}$}
	\While{not convergence}{
	\ForEach{batch $\mathcal{B}=(\mathcal{V}_\mathcal{B}, \mathcal{E}_\mathcal{B})$ from $\mathcal{\widetilde{G}}$}{
	$\bm{X}_\mathcal{B} := \Gamma(\mathcal{V}_\mathcal{B})$ \Comment*[r]{Eq.\eqref{equ:lstm}}
	$\bm{X}_\mathcal{B}:=\Psi(\bm{X}_\mathcal{B}, \mathcal{E}_\mathcal{B})$

	$\mathcal{L}=-\mathcal{I}_{\omega}(\Lambda^k(\mathcal{G}),\bm{X}_\mathcal{B})$ \Comment*[r]{Eq.\eqref{equ:jsd_mi}}

	$\mathcal{L}$.optimize()
	}
	}
\end{algorithm}

\subsection{Training Scheme and KG Inference}\label{sec:complete_alg}
\subsubsection{Training objective}
In this section, we briefly introduce our complete algorithm of \textit{NORAN} and present the training scheme in Algorithm~\ref{alg:LEIM}. After this, we obtain the model \(\Psi \circ \Gamma\) to infer expressive representations of inductive triples. 
Given the final relational embedding for an inductive triple learned from model \(\Psi \circ \Gamma\),
we conduct the link prediction with a classifier using \textit{logistic regression}, i.e. \(p(\bm{x}) = Sigmoid(\bm{w}^T\bm{x}+b)\). 

\subsubsection{Inductive inference and complexity analysis}
For inductive inference, given any triple \(t=(h,r,t)\) with an unseen entity $h$ or $t$, we first create a new node \(v=\Gamma(h,r,t)\) and update the original relation network \(\widetilde{\mathcal{G}}\) by adding node \(v\) to it. Then we predict the probability of triple existence with \(p\circ \Psi \circ \Gamma\).

In contrast with the \textit{Embedding-based Paradigm}, which requires retraining for inductive inference, our proposed framework only needs updating the relation network and then inferring with the trained model \(p\circ \Psi \circ \Gamma\). Formally, the inference time complexity of our proposed framework is:
\begin{equation*}
	\mathcal{O}(\underbrace{3bf^2}_{\Gamma} + \underbrace{bd^L f + bLf^2}_{\Psi}),
\end{equation*}
where \(b\) is the number of inductive triples; \(f\) is the hidden dimension assuming all embeddings (relation, entity, and triple) have the same dimension; \(d\) is the averaged node degree in relation network \(\widetilde{\mathcal{G}}\); and \(L\) is the number of layers of \(\Psi\). Notably, we omit the time complexity of \(p\) and relation network renewal, since both of which are smaller than the others by magnitudes. Therefore, the main inference time overhead lies on the message passing operation of \(\Psi\), which could exponentially grow with \(L\). However, based on our empirical results, a two-layer GNN could generally achieve the on-par performance given the naturality of the local view for GNNs. This is also consistent with the argument for deepening GNNs~\cite{chen2021bag}. 
\subsection{GNN Instantiations Selection}\label{sec:GNN_instantiations}

Our proposed framework is model-agnostic and compatible with all GNN frameworks. One can select the most effective GNN model by performing an exhausting search. However, it is expensive in terms of computing resources and not the focus of this work. In this section, we provide a valid conjecture on the selection of GNN models via a comprehensive theoretical analysis. Please note that our evaluation space focuses on intrinsic MP layers, regardless of other training tricks like residual connections and normalization~\cite{chen2021bag}. 
% \subsubsection{GNN Instantiations}
\subsubsection{General framework of MP layer}

Typically, based on the convolution matrix, there are two categories for representative MP layers: $(i)$ \textit{Fixed MP}: GCN~\cite{kipf2016semi}, GraphSAGE~\cite{hamilton2017inductive}, GIN~\cite{xu2018powerful}, SGC~\cite{wu2019simplifying}; $(ii)$ \textit{Learnable MP}: GAT~\cite{velivckovic2017graph}. The reformulation of typical MP layers can be found on Table~\ref{tab:framework_for_mp_1}. Recall that we actually train two GNN models during our training process: one is $\Psi$ that applied on \textit{relation network} for inductive inference; and the other is $\Omega$ that applied on \textit{logic evidence} for discriminating negative samples. We unify exactly the same architecture in our implementation for simplification.

\subsubsection{Influence distribution}
Our conjecture for model selection is motivated by a general question: \textit{what type of GNNs is capable of capturing the proper logic evidence for inductive inference}? Taking Fig.~\ref{fig:illustration} $(2)$ as an example, when we try to inference the relationship between \textit{Martin Eberhard} and \textit{California} (i.e. Node $8$), some existing logic evidence, \textit{Elon Musk $\xrightarrow{\text{:WorkIn}}$ TESLA $\xrightarrow{\text{:LocatedIn}}$ California} and \textit{Martin Eberhard $\xrightarrow{\text{:WorkIn}}$ TESLA} (i.e. Node $4$, $5$, and $7$) are more proper than the others. Therefore, in \textit{better} MP strategies, node $4$, $5$, and $7$ would have more \textit{influence} on the inference.
Related to ideas of sensitivity analysis and influence functions in statistics~\cite{xu2018representation}, we measure the sensitivity of node x to node y, or the influence of y on x with the concept of \textit{influence distribution}.
% This sheds light on our analysis of \textit{influence distribution}. Following \cite{xu2018representation}, we provide the definition of \textit{influence distribution} as follows.

\begin{definition}\label{def:influence_dist}
	\textbf{Influence Distribution}: Given a graph $\mathcal{G}=(\mathcal{V}, \mathcal{E})$, let $\bm{h}_v^{(i)}$ be the hidden features of node $v$ at $i$-th layer of a $k$-layer GNN model $\Psi$. The influence score of node $u$ on node $v$ is defined as $\mathcal{S}(u, v) = \sum_{i,j}\Bigg[\frac{\partial \bm{h}_u^{(k)}}{\partial \bm{h}_v^{(0)}}\Bigg]_{(i, j)}$. Then the influence distribution is defined as $\mathcal{S}_v (u) = \mathcal{S}(u, v) / \sum_{i \in \mathcal{V}} \mathcal{S}(i, v)$, which is the normalization of the corresponding influence score over all nodes.
\end{definition}

\subsubsection{Theoretical analysis}
Previous studies have verified that influence distributions of common aggregation schemes are closely connected to random walk distributions~\cite{xu2018representation}.
% To be clear,  the \textit{influence distribution} obtained by GraphSAGE~\cite{hamilton2017inductive} is expectedly equal to a random walk distribution. 
Extending the theorem to general MP frameworks, we can get a corollary as follows.
\begin{corollary}\label{cor:influence_dist}
	Given a graph $\mathcal{G}=(\mathcal{V}, \mathcal{E})$, let $\bm{h}_v^{(i)}$ be the hidden features of node $v$ at $i$-th layer of a $k$-layer GNN model $\Psi$ whose MP can be rewritten as $\bm{X}^{(i+1)} = \sigma\big( \mathcal{C}^{(i)}\bm{X}^{(i)}\circ f^{(k)} \big)$. If the convolution matrix is pre-fixed, i.e. $\frac{\partial \mathcal{C}}{\partial \bm{X}^{(i)}} = \bm{0}$, the influence distribution of any node $v$ is in expectation equal to a biased $k$-step random walk distribution, where the transition matrix is $\mathcal{T}$. $\mathcal{C} = (\bm{A}+\bm{I})\cdot \mathcal{T}$.
\end{corollary}
% Particularly, we provide the proof of Corollary~\ref{cor:influence_dist} as follows:
% \begin{proof}
% 	See \ref{app:proof}.
% \end{proof}
\begin{proof}
Let $f^{(l)}(\bm{x}) = f_m^{(l)} \circ f_{m-1}^{(l)} \circ \cdots \circ f_{1}^{(l)}$ be a $m$-layer MLP without learnable bias, where $f_m^{(l)}(\bm{x}) = \sigma(\bm{W}_m^{(l)}\bm{x})$. let $\mathcal{C} = (\bm{A}+\bm{I})\cdot \mathcal{T}$, where $\mathcal{T}$ is a pre-fixed weighted matrix. For any layer, by chain rule, we have
\begin{equation*}
 \frac{\partial \bm{h}_u^{(l)}}{\partial \bm{h}_v^{(0)}} =  \frac{\partial f_m^{(l)}}{\partial f_{m-1}^{(l)}} \cdot \frac{\partial f_{m-1}^{(l)}}{\partial f_{m-2}^{(l)}} \cdots \frac{\partial f_{2}^{(l)}}{\partial f_{1}^{(l)}} \cdot \sum_{i \in \widetilde{\mathcal{N}}(\bm{u})} \mathcal{T}_{u,i}\frac{\partial \bm{h}_i^{(l-1)}}{\partial \bm{h}_v^{(0)}},
\end{equation*}
where $\frac{\partial f_m^{(l)}}{\partial f_{m-1}^{(l)}} = \text{diag}\big(\bm{1}_{f_{m-1}^{(l)}>0}\big)\cdot \bm{W}_{m-1}^{(l)}$. Further more, according to chain rule and the message passing strategy, we get
\begin{equation*}
 \frac{\partial \bm{h}_u^{(k)}}{\partial \bm{h}_v^{(0)}} = \sum_{p \in \mathcal{P}}^{|\mathcal{P}|}\prod_{l=k}^{1} \mathcal{T}_{v_p^l, v_p^{l-1}} \frac{\partial f_m^{(l)}}{\partial f_{1}^{(l)}},
\end{equation*}
where $\mathcal{P}$ is the set of paths of length $k+1$ from node $u$ to node $v$. We can express an entry of the derivative for path $p$ as:
\begin{equation*}
  \bigg[\frac{\partial \bm{h}_u^{(k)}}{\partial \bm{h}_v^{(0)}}\bigg]^{(i,j)}_p = \bigg(\prod_{l=k}^1 \mathcal{T}_{v_p^l, v_{p-1}^l}\bigg) \bigg( \sum_q^{\mathcal{Q}} z_q\prod_{l=k}^1 \prod_{t=m}^1 w_{q,t}^{(l)} \bigg),
\end{equation*}
where $\mathcal{Q}$ is the set of paths $q$ from the input neuron $i$ to the output neuron $j$. $w_{q,t}^{(l)}$ denotes the neuron of $\bm{W}_m^{(l)}$ on path $q$. $z_q \in \{0, 1\}$ denotes whether the path is activated. Under the assumption that all $z_q$s obeys the same Bernoulli distribution, i.e. $P(z_q=1) = \beta$, we have:
\begin{equation*}
 \mathbb{E}\bigg[\frac{\partial \bm{h}_u^{(k)}}{\partial \bm{h}_v^{(0)}}\bigg] = \underbrace{\beta \cdot \prod_{l=k}^1 \prod_{t=m}^1 \bm{W}_t^{(l)}}_{u,v\text{ irrelevant}} \cdot \bigg( \sum_{p\in \mathcal{P}}^{|\mathcal{P}|}\prod_{l=k}^{1} \mathcal{T}_{v_p^l, v_{p-1}^l} \bigg).
\end{equation*}
As the former term is $u,v$ irrelevant, after normalization, the influence distribution of $\mathcal{S}_v(u)$ is $\sum_{p\in \mathcal{P}}^{|\mathcal{P}|}\prod_{l=k}^{1} \mathcal{T}_{v_p^l, v_{p-1}^l}$, which is exactly a bias random walk distribution whose transition matrix is $\mathcal{T}$.
\end{proof}

It indicates that for those MP layers whose convolution matrix is pre-fixed, the influence distribution for nodes is only \textit{structure-relevant} and irrespective to node features. For graphs that lay stress on the topological structures, such MP strategies are beneficial. However, for our \textit{relation network} whose node contains rich relational semantics is of vital importance for inductive KGC, a pre-fixed convolution matrix is not desired. In other words, MP layers with a learnable convolution matrix are probable to capture the logic evidence and learn more expressive representations for inductive KGC.
We evaluate such conjecture in Section~\ref{sec:main_results} by comparing several representative convolution matrices, $(i)$ \textit{learnable}: GAT~\cite{velivckovic2017graph}; $(ii)$ \textit{fixed}: GraphSAGE~\cite{hamilton2017inductive}, GIN~\cite{xu2018powerful}. The empirical results further demonstrate our conjecture to some extent.

\section{Experiments}\label{sec:experiments}
This section empirically evaluates the proposed framework and presents its performance on five KG datasets. Our empirical study is motivated by the following questions:
\begin{itemize}
	\item \textit{\textbf{Q1}}~(\ref{sec:main_results}): How does our proposed framework perform in comparison with the strongest baselines, including traditional embedding-based, rule-based methods and other MP-based methods?
	\item \textit{\textbf{Q2}}~(\ref{sec:ablation_relation_network}): W.r.t. Remark~\ref{remark: linking patterns} and Remark~\ref{remark: linking directions}, is our proposed \textit{relation network} a valid and effective modeling for inductive KG completion?
	\item \textit{\textbf{Q3}}~(\ref{sec:ablation_I}). Is our proposed \textit{logic evidence information maximization} a more effective training objective than standard negative sampling for inductive KGC?
\end{itemize}

\subsection{Experimental Setup}

\subsubsection{Datasets}
We conduct experiments on five standard KG benchmarks: $(i)$ \textbf{FB15K-237} that is a subset of FB15K where inverse relations are removed; $(ii) $\textbf{WN18RR} that is a subset of WN18 where inverse relations are removed; $(iii)$ \textbf{NELL995} that is extracted from the $995$-th iteration of the NELL system containing general knowledge; $(iv)$ \textbf{OGBL-WIKIKG2} that is extracted from the Wikidata knowledge base; and $(v)$ \textbf{OGBL-BIOKG} that is created using data from a large number of biomedical repositories. 
% The statistics of all datasets are shown in Table~\ref{tab:statistics} of the Appendix.

\begin{table*}[t]
	\centering
	\caption{Main results of inductive KGC on five benchmark datasets. We underline the best results within each of the three categories and bold NORAN's results that are better than all baselines.}
	\vspace{-4mm}
	\label{tab:comparison_results}
	\resizebox{1.0\textwidth}{!}{
		\begin{tabular}{clcccccccccccccccc}
			\toprule
			% \multicolumn{16}{c}{Precision@K}\\ \midrule
			\multirow{2}{*}{Categories} & Datasets                                & \multicolumn{3}{c}{FB15K-237} & \multicolumn{3}{c}{WN18RR} & \multicolumn{3}{c}{NELL995} & \multicolumn{3}{c}{	OGBL-WIKIKG2} & \multicolumn{3}{c}{OGBL-BIOKG}                                                                                                                                                                                                                                                                                \\
			\cmidrule(lr){2-2} \cmidrule(lr){3-5} \cmidrule(lr){6-8} \cmidrule(lr){9-11} \cmidrule(lr){12-14} \cmidrule(lr){15-17}
			                            & Metrics                                 & MRR                           & Hit@1                      & Hit@3                       & MRR                              & Hit@1                          & Hit@3                      & MRR                        & Hit@1                      & Hit@3                      & MRR                        & Hit@1                      & Hit@3                      & MRR                        & Hit@1                      & Hit@3   \\
			\midrule
			\multirow{5}{*}{Emb. Based}
			                            & TransE~\cite{bordes2013translating}
			                            & 0.289                                   & 0.198                         & 0.324                      & 0.265                       & 0.058                            & 0.445                          & 0.254                      & 0.169                      & 0.271                      & 0.213                      & 0.122                      & 0.229                      & 0.317                      & 0.26                       & 0.345                                \\
			                            & DistMult~\cite{yang2014embedding}
			                            & 0.241                                   & 0.155                         & 0.263                      & 0.430                       & 0.390                            & 0.440                          & 0.267                      & 0.174                      & 0.295                      & 0.199                      & 0.115                      & 0.210                      & 0.341                      & 0.278                      & 0.363                                \\
			                            & ComplEx~\cite{trouillon2016complex}
			                            & 0.247                                   & 0.158                         & 0.275                      & 0.440                       & 0.410                            & 0.460                          & 0.227                      & 0.149                      & 0.249                      & 0.236                      & 0.136                      & 0.254                      & 0.322                      & 0.257                      & 0.360                                \\
			                            & SimplE~\cite{kazemi2018simple}
			                            & \underline{0.338}                       & \underline{0.241}             & \underline{0.375}          & \underline{0.476}           & \underline{0.428}                & \underline{0.492}              & \underline{0.291}          & 0.198                      & \underline{0.314}          & 0.220                      & 0.131                      & 0.239                      & 0.319                      & 0.246                      & 0.358                                \\
			                            & QuatE~\cite{zhang2019Quaternion}
			                            & 0.319                                   & \underline{0.241}             & 0.358                      & 0.446                       & 0.382                            & 0.478                          & 0.285                      & \underline{0.201}          & 0.307                      & \underline{0.248}          & \underline{0.139}          & \underline{0.262}          & \underline{0.363}          & \underline{0.294}          & \underline{0.382}                    \\
			\cmidrule(lr){2-17}
                \multirow{2}{*}{Rule Based}
                                        &RuleN~\cite{meilicke2018fine}    &0.453       &0.387      &0.491      &0.514      &0.461      &0.532      &0.346      &0.279      &0.366 &- &- &- &- &- &- \\
                                        &DRUM~\cite{sadeghian2019drum}     &0.447       &0.373      &0.478      &0.521      &0.458      &0.549      &0.340      &0.261      &0.363 &- &- &- &- &- &- \\
                \cmidrule(lr){2-17}
			\multirow{7}{*}{MP based}
			                            & RGCN~\cite{schlichtkrull2018modeling}
			                            & 0.427                                   & 0.367                         & 0.451                      & 0.501                       & 0.458                            & 0.519                          & 0.329                      & 0.256                      & 0.348                      & 0.285                      & 0.176                      & 0.324                      & 0.381                      & 0.319                      & 0.399                                \\
			                            & RGHAT~\cite{zhang2020relational}
			                            & 0.440                                   & 0.361                         & 0.483                      & 0.518                       & 0.460                            & 0.540                          & 0.337                      & 0.274                      & 0.351                      & 0.301                      & 0.192                      & 0.329                      & 0.395                      & 0.334                      & 0.418                                \\
			                            & GraIL~\cite{teru2020inductive}
			                            & 0.465                                   & 0.389                         & 0.482                      & 0.512                       & 0.453                            & 0.539                          & \underline{0.355}          & \underline{0.282}          & 0.367                      & 0.327                      & 0.201                      & 0.336                      & 0.434                      & 0.379                      & 0.451                                \\
                               &ConGLR~\cite{lin2022incorporating} &0.463& 0.402& 0.483& 
                                0.512& 0.452& 0.541
                               & 0.352& 0.276& 0.366& 0.318& 0.219& 0.338& 0.422& 0.365& 0.431\\

			                            & PATHCON~\cite{wang2021relational}
			                            & \underline{0.483}                       & \underline{0.425}             & \underline{0.499}          & \underline{0.522}           & \underline{0.462}                &0.546           & 0.349                      & 0.276                      & \underline{0.369}          & \underline{0.339}          & \underline{0.243}          & \underline{0.347}          & \underline{0.457}          & \underline{0.395}          & \underline{0.472}                    \\
                               % https://github.com/zjukg/RMPI
                               &RMPI~\cite{geng2023relational} & 0.459& 0.396& 0.480 &
                               0.514& 0.454& 0.544
                               & 0.339& 0.277& 0.365& 0.313& 0.213& 0.336& 0.414& 0.358& 0.421\\
                               
                               & MBE~\cite{cui2022inductive} &0.477 &0.410 &0.495 &0.519 &0.451 &\underline{0.549}  &0.344 &0.270 &0.359 &0.331 &0.234 &0.345 &0.452 &0.380 &0.470 \\
			\cmidrule(lr){2-17}
			\multirow{3}{*}{Ours$^a$}
			                            & NORAN(GS)
			                            & \textbf{0.483}                          & \textbf{0.440}                & \textbf{0.504}             & \textbf{0.535}              & \textbf{0.471}                   & \textbf{0.564}                 & \textbf{0.364}             & \textbf{0.298}             & \textbf{0.381}             & \textbf{0.349}             & 0.237                      & \textbf{0.361}             & 0.454                      & \textbf{0.395}             & \textbf{0.479}                       \\
			                            & NORAN(GIN)
			                            & \textbf{\underline{0.499}}              & \textbf{\underline{0.451}}    & \textbf{\underline{0.519}} & \textbf{0.530}              & \textbf{0.467}                   & \textbf{0.560}                 & \textbf{0.370}             & \textbf{0.308}             & \textbf{0.379}             & \textbf{0.353}             & \textbf{0.251}             & \textbf{0.367}             & \textbf{0.469}             & \textbf{0.407}             & \textbf{0.492}                       \\
			                            & NORAN(GAT)
			                            & 0.468                                   & 0.422                         & 0.489                      & \textbf{\underline{0.540}}  & \textbf{\underline{0.499}}       & \textbf{\underline{0.575}}     & \textbf{\underline{0.374}} & \textbf{\underline{0.310}} & \textbf{\underline{0.392}} & \textbf{\underline{0.358}} & \textbf{\underline{0.260}} & \textbf{\underline{0.371}} & \textbf{\underline{0.475}} & \textbf{\underline{0.411}} & \textbf{\underline{0.495}}           \\ \cmidrule(lr){2-17}
			\textit{NORAN - Emb.}$^b$   & +2.2\% (Avg.)                           & +1.6\%                        & +2.6\%                     & +2.0\%                      & +1.8\%                           & +3.7\%                         & +2.9\%                     & +1.9\%                     & +2.8\%                     & +2.3\%                     & +1.9\%                     & +1.7\%                     & +2.4\%                     & +1.8\%                     & +1.6\%                     & +2.3\%  \\
			\textit{NORAN - GNN}$^b$    & +11.2\% (Avg.)                          & +16.1\%                       & +21.0 \%                   & +14.4\%                     & +6.4\%                           & +7.1\%                         & +8.3\%                     & +8.3\%                     & +10.9\%                    & +7.8\%                     & +11.0 \%                   & +12.1\%                    & +10.9\%                    & +11.2\%                    & +11.7\%                    & +11.5\% \\
			\bottomrule
			\multicolumn{14}{l}{$^a$ We test our NORAN with three backbones: GraphSAGE, GIN, and GAT, denoted as NORAN(GS), NORAN(GIN), NORAN(GAT), respectively. }                                                                  \\
			\multicolumn{14}{l}{$^b$ We show the margin between the best results of NORAN and the ones of the two branches of baselines methods. }
		\end{tabular}
	}
 \vspace{-3mm}
	\label{tab:main_results}
\end{table*}

\subsubsection{Baselines}
To valid the effectiveness of NORAN, we compare it with the state-of-art baselines according to our best knowledge. Following the taxonomy aforementioned in Section~\ref{sec:related_work}, we divide all baselines into three categories: $(i)$ \textit{embedding-based}: \textbf{TransE}~\cite{bordes2013translating}, \textbf{DistMult}~\cite{yang2014embedding}, \textbf{ComplEx}~\cite{trouillon2016complex}, \textbf{SimplE}~\cite{kazemi2018simple}, and QuatE~\cite{zhang2019Quaternion}; $(ii)$ \textit{Rule-based}: \textbf{RuleN}~\cite{meilicke2018fine} and \textbf{DRUM}~\cite{sadeghian2019drum}; $(iii)$ \textit{MP-based}: \textbf{RGCN}~\cite{schlichtkrull2018modeling}, \textbf{RGHAT}~\cite{zhang2020relational}, \textbf{GraIL}~\cite{teru2020inductive}, 
\textbf{ConGLR}~\cite{lin2022incorporating}, 
\textbf{PATHCON}~\cite{wang2021relational}, 
\textbf{RMPI}~\cite{geng2023relational}
and \textbf{MBE}~\cite{cui2022inductive}. Particularly, the five embedding based methods are not naturally inductive and we retrain the embedding table for inductive inference. For \textit{MP-based} methods, we have selected the most effective ones for KGC in the same setting. Notably, we did not include several related works~\cite{ontologyschema,lee2023ingram,jininductive} in our experiments since their original source codes were written only for KGC with unseen relations, while our model is designed for KGC with unseen entities.

\subsubsection{Implimentation details}
The aforementioned five benchmark datasets were originally developed for the \textit{transductive setting}. In other words, the entities of the standard test splits are a subset of the entities in the training splits. In order to facilitate inductive testing, we create new inductive benchmark datasets by sampling disjoint subgraphs from the KGs in these datasets. Specifically, we first sample a subset of entities from the test set at random, then remove them and their associated edges from the training set. The remaining training set is used to train the models, and the removed edges are added back in during evaluation.
We use \textbf{MRR} (mean reciprocal rank) and \textbf{Hit@1, 3} as evaluation metrics. 
% More implementation details can be found at Sction~\ref{app:implementaion_details} of Appendix.
% \subsubsection{Feature Initialization}
% To perform inductive KG completion, we use xavier\_norm initializer to initialize the feature matrix of all entities and relations, including unseen entities. It should be noted that the embedding representations of unseen entities were not updated during the training stage since all associated triples of these unseen entities in the training set were masked during the data pre-processing phase. Especially, GraIL only relies on edge features to conduct inductive kg reasoning while PATHCON adopted the node labeling technique to generate the node feature matrix without manual Initialization.

% \subsubsection{Searched Hyperparameters}
All baseline methods used in this paper are based on their open-source implementations. All models are implemented in Pytorch and trained on an RTX 3090 with 24 RAM. To make a fair comparison, for all models, the embedding size of entities and relations is set as $100$ for the input, latent, and output layer. Exceptionally, PATHCON, which adopted the node labeling technique to generate the node feature matrix, has a distinctive size of input embedding that varies with datasets. We optimize all models with Adam optimizer, where the batch size is fixed at $256$. Specially, we run all models on the large OGBL dataset with a batch size of $32$ to avoid out-of-memory. We use the default Xavier initializer to initialize model parameters, and the initial learning rate is $0.005$. We apply a grid search for hyperparameter tuning. The learning rate is searched from $\{0.05, 0.01, 0.005, 0.001\}$,   the margin parameter used in all baselines, and our negative sampling training scheme from $0$ to $1$. We use their best configuration as the default value for other distinctive hyperparameters in the baseline methods. To reduce the randomness, we use the random seed and report the average results of three runs.

\subsection{Main Results: Q1}\label{sec:main_results}
To answer Q1, we perform extensive experiments to compare NORAN with the \textit{strongest} baselines best to our knowledge. The comparison results are placed in Table~\ref{tab:main_results}. We summarize our experimental observations as follows.

\noindent \textit{Obs. 1. NORAN significantly outperforms  SOTA KGC baselines.} We compare NORAN (ours) with $5$ representative embedding-based methods, $2$ rule-based methods and $7$ SOTA MP-based methods on $5$ KG datasets. As shown in Table~\ref{tab:main_results}, NORAN significantly outperforms all other baselines in all evaluation metrics (i.e., MRR, Hit@1, Hit@3). Particularly, we \textbf{bold} NORAN's results that are better than all baseline. We evaluate NORAN with three commonly-used backbones, and the bold results notably take the vast majority, indicating the effectiveness of our proposed NORAN. Besides, we compute the margin between the best results (\underline{underlined}) of NORAN and the two branches of baselines by column. Notably, NORAN outperforms the best embedding-based (MP-based) baseline with a large margin, averaging $11.2\%$ ($2.2\%$).

\noindent
\textit{Obs. 2. NORAN integrating GAT (learnable convolution matrix) generally outperforms the MP layers with pre-fixed convolution matrices.} We test NORAN with three backbones: GAT, GraphSAGE, and GIN, where the first one's convolution matrix is learnable and the last two pre-fix their convolution matrices. As shown in Table~\ref{tab:main_results}, GAT achieves the best results (\underline{underlined}) on four datasets while GraphSAGE and GIN have an on-par performance. 
% These results verify our Conjecture~\ref{con:gnn_model_selection} in Appendix to some extent.

\begin{table*}[t]
	\centering
	\caption{Results of ablation study for relation network construction rules on three benchmark datasets.}
	\label{tab:ablation_relation_network_app}
 \vspace{-4mm}
	\resizebox{.9\textwidth}{!}{
		\begin{tabular}{clcccccccccc}
			\toprule
			% \multicolumn{16}{c}{Precision@K}\\ \midrule
			\multirow{2}{*}{Construction Rule} & Datasets   & \multicolumn{3}{c}{WN18RR} & \multicolumn{3}{c}{NELL995} & \multicolumn{3}{c}{	OGBL-WIKIKG2}                                                 \\
			\cmidrule(lr){2-2} \cmidrule(lr){3-5} \cmidrule(lr){6-8} \cmidrule(lr){9-11}
			                                   & Metrics    & MRR                        & Hit@1                       & Hit@3                            & MRR   & Hit@1 & Hit@3 & MRR   & Hit@1 & Hit@3 \\
			\midrule
			\multirow{1}{*}{Default}
			                                   & NORAN(GAT)
			                                   & 0.540      & 0.499                      & 0.575                       & 0.374                            & 0.310 & 0.392 & 0.358 & 0.260 & 0.371         \\
			\cmidrule(lr){2-11}

			\multirow{1}{*}{w.o. head-head}
			                                   & NORAN(GAT)
			                                   & 0.521      & 0.457                      & 0.558                       & 0.348                            & 0.273 & 0.375 & 0.329 & 0.230 & 0.341         \\
			\cmidrule(lr){2-11}

			\multirow{1}{*}{w.o. tail-tail}
			                                   & NORAN(GAT)
			                                   & 0.517      & 0.454                      & 0.549                       & 0.351                            & 0.288 & 0.364 & 0.327 & 0.224 & 0.342         \\
			\cmidrule(lr){2-11}

			\multirow{1}{*}{w.o. head-tail}
			                                   & NORAN(GAT)
			                                   & 0.501      & 0.446                      & 0.532                       & 0.339                            & 0.275 & 0.359 & 0.314 & 0.212 & 0.319         \\

			\bottomrule
		\end{tabular}}
\vspace{-2mm}
\end{table*}

\begin{table*}[t]
	\centering
	\caption{Ablation study of training objective, i.e., \textit{Logic Evidence Information Maximization} (LEIM)}
	\label{tab:ablation_I_app}
 \vspace{-4mm}
		\resizebox{0.9\textwidth}{!}{
	\begin{tabular}{clcccccccccc}
		\toprule
		% \multicolumn{16}{c}{Precision@K}\\ \midrule
		\multirow{2}{*}{Training Objective} & Datasets            & \multicolumn{3}{c}{WN18RR} & \multicolumn{3}{c}{NELL995} & \multicolumn{3}{c}{	OGBL-WIKIKG2}                                                 \\
		\cmidrule(lr){2-2} \cmidrule(lr){3-5} \cmidrule(lr){6-8} \cmidrule(lr){9-11}
		                                   & Metrics             & MRR                        & Hit@1                       & Hit@3                            & MRR   & Hit@1 & Hit@3 & MRR   & Hit@1 & Hit@3 \\
		\midrule
		\multirow{3}{*}{\shortstack{Naive                                                                                                                                                            \\ Negative Sampling}}
		                                   & NORAN(GS)
		                                   & 0.529               & 0.469                      & 0.554                       & 0.351                            & 0.290 & 0.357 & 0.324 & 0.215 & 0.322         \\
		                                   & NORAN(GIN)
		                                   & 0.537               & 0.490                      & 0.567                       & 0.357                            & 0.291 & 0.374 & 0.335 & 0.231 & 0.349         \\
		                                   & NORAN(GAT)
		                                   & 0.534               & 0.481                      & 0.562                       & 0.361                            & 0.299 & 0.377 & 0.342 & 0.238 & 0.355         \\

		\cmidrule(lr){2-11}

		\multirow{3}{*}{LEIM (JSD)}
		                                   & NORAN(GS)
		                                   & 0.535               & 0.471                      & 0.564                       & 0.364                            & 0.298 & 0.381 & 0.349 & 0.237 & 0.361         \\
		                                   & NORAN(GIN)
		                                   & 0.530               & 0.467                      & 0.560                       & 0.370                            & 0.308 & 0.379 & 0.353 & 0.251 & 0.367         \\
		                                   & NORAN(GAT)
		                                   & 0.540               & 0.499                      & 0.575                       & 0.374                            & 0.310 & 0.392 & 0.358 & 0.260 & 0.371         \\
		\cmidrule(lr){2-11}

		\multirow{3}{*}{LEIM (InfoNCE)}
		                                   & NORAN(GS)
		                                   & 0.532               & 0.476                      & 0.558                       & 0.358                            & 0.294 & 0.369 & 0.331 & 0.220 & 0.343         \\
		                                   & NORAN(GIN)
		                                   & 0.541               & 0.495                      & 0.583                       & 0.365                            & 0.301 & 0.376 & 0.339 & 0.231 & 0.359         \\
		                                   & NORAN(GAT)
		                                   & 0.544               & 0.506                      & 0.581                       & 0.367                            & 0.306 & 0.383 & 0.337 & 0.232 & 0.352         \\

		\bottomrule
	\end{tabular}
 \vspace{-5mm}
		}
\end{table*}

\subsection{Ablation Study on Relation Network: Q2}\label{sec:ablation_relation_network}
According to the definition of the relation network, there are two latent construction particulars corresponding to two construction issues: $(i)$ the semantics of various `\textit{entity sharing}' patterns (Remark~\ref{remark: linking patterns}) and $(ii)$ the semantics of the `\textit{linking direction}' (Remark~\ref{remark: linking directions}). We conduct a detailed ablation study in this section to validate our arguments in Remark~\ref{remark: linking patterns} and ~\ref{remark: linking directions}.
\begin{remark}\label{remark: linking patterns}
	\textbf{Linking pattern}. For any two triples sharing entities, i.e. $T_1=(h_1, r_1, t_1) \cap T_2=(h_2, r_2, t_2)$, we have three linking patterns: $(i)$ head-head sharing~($h_1 = h_2)$, $(ii)$ tail-tail sharing~($t_1 = t_2$), and $(iii)$ head-tail sharing~($t_1 = h_2 \oplus t_2=h_1$). For our construction criterion, we build links for any of the three linking patterns based on the rationale that the three patterns possess different semantics.
\end{remark}
\begin{remark}\label{remark: linking directions}
	\textbf{Linking direction}. According to Def.~\ref{def:relational_graph}, for different linking patterns discussed in Remark~\ref{remark: linking patterns}, we uniformly set the according edges in relation network as bidirectional edges, omitting the linking directions. The rationale here is that the semantics of relations in knowledge graphs are not direction-sensitive. For instance, a typical triple \textit{(Elon Musk, :WorkIn, TESLA)} in Figure~\ref{fig:illustration} (i) is semantically identical to \textit{(TESLA, :Foundedby, Elon Musk)}.
\end{remark}
% \subsubsection{Ablation study for relation network construction: Q2}\label{sec:ablation_relation_network}

To answer Q2, we perform a thorough ablation study for relation network construction. For remark~\ref{remark: linking patterns}, recalling that \textit{relation networks} have three patterns: \textit{head-head sharing}, \textit{head-tail sharing}, and \textit{tail-tail sharing}, we iteratively ablate each pattern to construct the corresponding relation network and then use it to train our model. We do this ablation study using GAT as the backbone. The empirical results are shown on Table~\ref{tab:ablation_relation_network_app}. As we can see, ablating any of the three patterns from relation networks would markedly damage the performance of NORAN, indicating the validness of our construction rule for relation networks.

\noindent \textit{Obs. 3. Head-tail pattern weights more than the others for relation network construction.} Particularly, we observe that ablating head-tail patterns would reduce the performance with a large margin in comparison with the other ablations. Actually, the head-tail pattern corresponds to the translation-based logic rules and dominates multi-hop knowledge graph reasoning, while the other two corresponds, i.e., \textit{head-head} and \textit{tail-tail} logically correspond to the union and intersection operations, respectively, as defined in complex knowledge graph reasoning~\cite{ren2020query2box}. Admittedly, the experiment results indicate that the translation-based logic evidence (head-tail pattern) are crucial for inductive inference, which is actually the same in logic with multi-hop knowledge graph reasoning. However, according to our results, the other logic operation (\textit{Union} and \textit{Intersection}) can also facilitate KGC, which is nevertheless ignored by translation-based embedding methods, which, in a way, explains their relatively worse inductive performance as shown in Table~\ref{tab:main_results}.

\subsection{Ablation Study for Training Objective: Q3}\label{sec:ablation_I}
% $\mathcal{I}_{\omega, \psi}$
    To answer Q3, we perform an ablation study to demonstrate the effectiveness of LEIM. Specifically, we introduce another InfoNCE estimator~\cite{hjelm2018learning} which is defined as follows:
\begin{multline}\label{equ:infonce_mi}\small
	\mathcal{I}^{InfoNCE}_{\omega, \psi}(\Lambda, \Psi) = \mathbb{E}_\mathbb{P}\Bigg(T_{\omega, \psi}(\Lambda^k(v), \Psi^k(v)) \\ - \mathbb{E}_{\mathbb{P}'}\Big(\log \sum_{v'\sim\mathbb{P}'}e^{T_{\omega, \psi}(\Lambda^k(v'), \Psi^k(v))}\Big)\Bigg).
	% \vspace{3mm}
\end{multline}
Then we evaluate LEIM with two MI estimators, i.e., JSD (our default) and InfoNCE, and compare them to the naive negative sampling (NS) training objective, which is the common practice in previous works~\cite{wang2021relational,schlichtkrull2018modeling}. Selected experiment results are on Table~\ref{tab:ablation_I_app}. As shown in Table~\ref{tab:ablation_I_app}, we bold the best results by column. Obviously, LEIM generally outperforms Naive NS with a large margin for all three backbones, demonstrating the effectiveness of our proposed training objective LEIM. The key difference between JSD and InfoNCE is how \textit{negative sampling} is performed. In practice, for JSD, we perform a $1:1$ negative sampling within a batch and compute the expectation of the data distribution (positive samples) and a noise distribution (negative samples) separately. For InfoNCE --- that is a typical \textit{contrastive learning} style --- it requires numerous negative samples to be competitive, where the common practice is utilizing all the other instances with a batch \(\mathcal{B}\) as the negative samples. For a batch size \(|\mathcal{B}|\), it gives \(\mathcal{O}(|\mathcal{B}|)\) negative samples per positive sample. Considering the computing efficiency and our empirical evaluation in \ref{sec:ablation_I}, we choose the JSD as our default instantiation. In addition, for the inter-comparison among training objectives, we found one interesting phenomenon:

\noindent
\textit{Obs. 4. GAT generally benefits the most from LEIM.} On three datasets, Through LEIM, GAT gains a performance increase by an averaging margin of $1.2\%$ in MRR and $1.5\%$ in Hit@3, which is larger than the other backbones. Specifically, On WN18RR, GAT performs worst among three training objectives when Naive NS trains them. However, when applying LEIM, GAT outstandingly achieves the best result. In contrast, GIN with a prefixed convolution matrix even has a negative boosting when trained with LEIM.

\section{Conclusions and Future work}

Inductive KG completion tries to infer relations for newly-coming entities is more inline with reality as real-world KGs are always evolving and introducing new knowledge.  In this work, we proposed a novel message-passing framework, i.e., NORAN, which aims to mine latent relation semantics for inductive KG completion. Initially, we propose a distinctive perspective on KG modeling termed \textit{relation network}, which considers relations as indicators of logical rules. Moreover, to preserve such logical semantics inherent in KGs, we propose an innovative informax training objective, \textit{logic evidence information maximization} (LEIM). The combination of the relation network and LEIM forms our NORAN framework for inductive KGC.
Since our framework is model-agnostic, we also provide a detailed theoretical analysis to explore which kind of MP model is more effective for the relation network and provide guidelines
for model selection to enable inductive reasoning.
Extensive experiment results on five KG benchmarks demonstrate the superiority of our proposed
NORAN over state-of-the-art competitors.
As part of our future work, we intend to explore the potential applications of our approach in real-world scenarios, such as KG-enhanced recommendation systems and question-answering.

\section{Acknowledgements}
 The work described in this paper was fully supported by a grant from the Research Grants Council of the Hong Kong Special Administrative Region, China (Project No. PolyU 25208322).
% \clearpage
\balance 
\bibliographystyle{unsrt}
\bibliography{main.bib}

\end{document}